\documentclass{article}

 \usepackage[preprint]{neurips_2026}

\usepackage[utf8]{inputenc} 
\usepackage[T1]{fontenc}    
\usepackage{hyperref}       
\usepackage{url}            
\usepackage{booktabs}       
\usepackage{amsfonts}       
\usepackage{nicefrac}       
\usepackage{microtype}      
\usepackage{xcolor}         

\definecolor{metacolor}{HTML}{0064E0}
\definecolor{myred}{HTML}{EA4335}
\definecolor{mygreen}{HTML}{34A853}
\definecolor{myblue}{HTML}{4285F4}
\definecolor{myyellow}{HTML}{FBBC04}

\newcommand{\ie}{\textit{i.e.}\xspace}
\newcommand{\eg}{\textit{e.g.}\xspace}

\usepackage{colortbl}
\usepackage{multirow}
\usepackage[linesnumbered,ruled,vlined]{algorithm2e}
\usepackage{soul}
\usepackage{makecell}
\usepackage{wrapfig}
\usepackage{cancel}
\usepackage{graphicx}
\usepackage{subfig}
\usepackage{enumitem}
\usepackage{subcaption}
\usepackage{float}

\usepackage{xcolor}
\usepackage{xcolor}
\usepackage[most]{tcolorbox}
\tcbuselibrary{skins, breakable}

\definecolor{mygray}{RGB}{240, 240, 240} 
\definecolor{myblue}{RGB}{230, 240, 255}
\definecolor{mypink}{RGB}{250, 230, 235}
\definecolor{myred}{RGB}{172, 59, 40}
\definecolor{mygreen}{RGB}{141, 185, 89}
\definecolor{mycharpink}{RGB}{203, 91, 96}
\definecolor{mycharblue}{RGB}{93, 114, 188}
\newtcbox{\mybox}[1]{%
  enhanced, on line, arc=2pt, 
  colback=#1, 
  colframe=#1!70!black, 
  boxrule=0.5pt, boxsep=0pt, left=2pt, right=2pt, top=1pt, bottom=1pt,
  frame style={dash pattern=on 3pt off 1.5pt}
}

\usepackage{amsmath, amssymb, amsthm}

\newtheorem{lemma}{Lemma}[section]

\newtheorem{theorem}{Theorem}

\usepackage{geometry}
\usepackage{bm}
\usepackage{xcolor}
\usepackage{mdframed}

\DeclareMathOperator{\Var}{Var}

\title{Think, then Score: Decoupled Reasoning and Scoring for Video Reward Modeling}

\author{
    Yuan Wang$^{1, 2}$\footnote[1]\ \ \footnote[4]\ \ , 
    Ouxiang Li$^{1}$ \footnote[4] \ \ ,
    Yulong Xu$^2$\footnote[3]\ \ , 
    Borui Liao$^2$, 
    Jiajun Liang$^{2}$\footnote[2]\ \ ,\\
    Jinghan Li$^1$,
    Meng Wang$^2$,
    Xintao Wang$^2$,
    Pengfei Wan$^2$,
    Kuien Liu$^3$,
    Xiang Wang$^1$\footnote[2]\ \ \\
    \small $^1$University of Science and Technology of China, 
    $^2$Kling Team, Kuaishou Technology,  \\
   \small $^3$ Institute of Software Chinese Academy of Sciences \\
    {\tt\small wy1001@mail.ustc.edu.cn, lioox@mail.ustc.edu.cn, liangjiajun@kuaishou.com}
}

\begin{document}

\renewcommand{\thefootnote}{\fnsymbol{footnote}} 
\footnotetext[1]{Work done during internship at Kling Team, Kuaishou Technology.}
\footnotetext[4]{Equal Contribution.}
\footnotetext[2]{Corresponding authors.}
\footnotetext[3]{Project Lead.}
\renewcommand{\thefootnote}{\arabic{footnote}} 

\maketitle



\vspace{-5mm}
\begin{figure*}[ht]
\centering
\includegraphics[width=\textwidth]{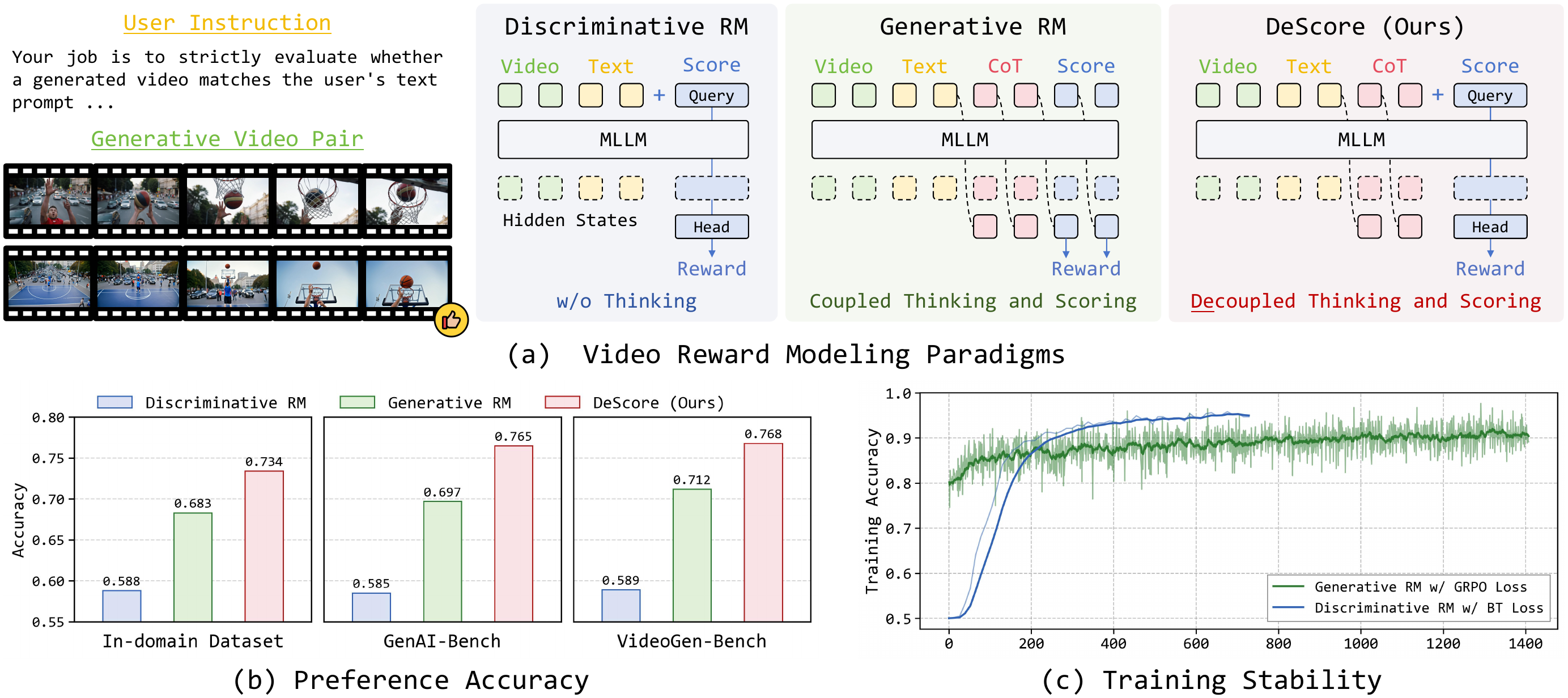}
\vspace{-5mm}
\caption{\textbf{Overview and Motivation of DeScore.} \textbf{(a) Video Reward Modeling Paradigms.} Existing video reward models generally follow two paradigms: Discriminative RMs directly regress rewards without explicit thinking (\eg, CoT), and Generative RMs couple thinking and scoring within a single autoregressive sampling chain. DeScore improves both paradigms based on two observations: First, \textbf{(b) Preference Accuracy} shows that incorporating CoT enables Generative RMs to outperform Discriminative RMs, highlighting the necessity of explicit thinking for generalization. Second, \textbf{(c) Training Stability} reveals that coupling thinking and scoring requires the final score to be optimized through GRPO loss~\cite{DeepSeekR1, DeepSeekMath}, leading to pronounced training fluctuations. In contrast, discriminative training with BT loss~\cite{bradley1952rank} exhibits smooth convergence. Motivated by these findings, DeScore introduces a decoupled ``think-then-score'' paradigm that effectively leverages the generalization benefits of CoT reasoning while preserving the training stability inherent to discriminative optimization.}
\label{fig:intro}
\end{figure*}

\begin{abstract}
Recent advances in generative video models are increasingly driven by post-training and test-time scaling, both of which critically depend on the quality of video reward models (RMs). An ideal reward model should predict accurate rewards that align with human preferences across diverse scenarios. However, existing paradigms face a fundamental dilemma: \textit{Discriminative RMs} regress rewards directly on features extracted by multimodal large language models (MLLMs) without explicit reasoning, making them prone to shortcut learning and heavily reliant on massive data scaling for generalization. In contrast, \textit{Generative RMs} with Chain-of-Thought (CoT) reasoning exhibit superior interpretability and generalization potential, as they leverage fine-grained semantic supervision to internalize the rationales behind human preferences. 
However, they suffer from inherent optimization bottlenecks due to the coupling of reasoning and scoring within a single autoregressive inference chain. 
To harness the generalization benefits of CoT reasoning while mitigating the training instability of coupled reasoning and scoring, we introduce DeScore, a training-efficient and generalizable video reward model. DeScore employs a decoupled ``think-then-score'' paradigm: an MLLM first generates an explicit CoT, followed by a dedicated discriminative scoring module consisting of a learnable query token and a regression head that predicts the final reward.
DeScore is optimized via a two-stage framework: (1) a discriminative cold start incorporating a random mask mechanism to ensure robust scoring capabilities, and (2) a dual-objective reinforcement learning stage that independently refines CoT reasoning quality and calibrates the final reward, ensuring that higher-quality reasoning directly translates to superior model performance.
Empirical evaluations demonstrate that DeScore achieves superior training efficiency and optimization stability, while outperforming state-of-the-art methods across diverse in-domain and out-of-distribution benchmarks. Moreover, DeScore also proves effective for post-training, leading to improved generated video quality.
\end{abstract}

\section{Introduction}
\label{sec:intro}
Modern generative video models~\cite{kling2026, hailuo2025, seedream2025, sora2025, seedance2026, wan2025wan, team2025longcat, kong2024hunyuanvideo} have made remarkable progress in high-quality video synthesis, largely driven by post-training~\cite{liu2025flow, xue2025dancegrpo, wallace2024diffusion} and test-time scaling~\cite{ma2025inference, oshima2025inference}. Crucial to these paradigms is the video reward model, whose quality dictates the performance ceiling of the optimization process. An ideal video reward model must accurately align with human preferences across diverse scenarios and complex motion patterns. This necessitates robust out-of-distribution (OOD) generalization to maintain the accuracy of the reward signals.

One representative category of video reward models is the discriminative paradigm~\cite{he2024videoscore, liu2025improving}, which typically regresses scalar rewards from multimodal large language model (MLLM) features. Despite the stable optimization signals offered by regression losses (\eg, Bradley-Terry (BT) loss or MSE loss), the absence of explicit reasoning forces these models to infer fine-grained semantic differences from coarse preference labels. This often leads to shortcut learning~\cite{zeng2024dawn}, where models exploit shortcut features to fit training labels, rather than capturing the intrinsic semantic attributes aligned with human judgment. Compensating for this requires massive data scaling~\cite{liu2025improving}, which not only incurs prohibitive training overhead but also limits the model's adaptability to diverse OOD scenarios.

Another representative category follows the generative paradigm~\cite{wu2025rewarddance, he2025videoscore2, wang2025unified, wang2025unified2, wang2024lift, xu2026visionreward}, formulating reward modeling as a next-token prediction task within an MLLM framework. While directly generating a score token shares the limitations of discriminative models, advanced methods incorporate Chain-of-Thought (CoT) reasoning~\cite{wu2025rewarddance, he2025videoscore2, wang2025unified, wang2024lift} prior to the final reward. This process provides fine-grained semantic supervision, enabling the model to internalize the rationale behind human preferences. Specifically, the model learns why a video is superior rather than merely fitting a ranking, thereby enhancing its generalization potential, as evidenced by Figure~\ref{fig:intro} (b).

However, the generative paradigm predicts the reward as a token sequence in a next-token prediction manner rather than as an explicit scalar, which incurs the following optimization bottlenecks in Figure~\ref{fig:intro} (c).
\textbf{(1) \textit{Lack of direct reward-value optimization}:} Generative video reward models rely on supervised fine-tuning (SFT) and reinforcement learning (RL) for optimization. These methods fundamentally optimize discrete token probabilities instead of providing a direct gradient for the reward value, compared to the BT loss~\cite{bradley1952rank} (see Appendix~\ref{supp:direct}). 
Additionally, coupling logical reasoning and scoring within a single sampling chain forces a heavy reliance on RL (\eg, GRPO~\cite{DeepSeekR1, DeepSeekMath}) to improve performance, introducing two primary challenges:
\textbf{(2) \textit{Credit assignment difficulty}:} When an entire generated sequence shares a single rollout reward, it becomes difficult to determine whether a suboptimal output stems from low-quality intermediate reasoning tokens or an inaccurate final reward token. 
\textbf{(3) \textit{High-variance policy gradients}:} RL-based policy optimization inherently suffers from high gradient variance~\cite{zhang2021sample, he2025vl, yu2025dapo}, which leads to training instability (see Appendix~\ref{supp:gradient}).

These challenges motivate a fundamental design question: \textit{\textbf{{How can we harness the interpretability and generalization introduced by CoT reasoning during reward modeling while shielding the training process from the optimization instability of a coupled sampling chain?}}} To this end, we introduce DeScore, a training-efficient and generalizable video reward model through a decoupled ``Think-then-Score'' paradigm. By isolating reasoning from scoring, DeScore retains the fine-grained interpretability of generative CoT while mitigating the aforementioned bottlenecks through a specialized discriminative scoring module consisting of a learnable query token and a regression head. Crucially, this structural decoupling enables targeted optimization for the scoring module, bypassing the credit assignment dilemma caused by applying GRPO across the entire reasoning sequence. Moreover, the final scalar reward can be directly optimized via a stable, margin-based loss (\eg, BT loss) rather than relying on high-variance policy gradients, thereby ensuring robust training efficiency.

To facilitate effective reward model training with our decoupled design, we instantiate DeScore based on Qwen3-VL-8B~\cite{bai2025qwen3} and propose a two-stage training framework: (1) discriminative cold start and (2) dual-objective reinforcement learning (RL).
In the cold-start stage, we jointly fine-tune the MLLM backbone and the scoring module using the BT loss. To improve robustness, we introduce a random masking mechanism that randomly drops the CoT during training. This strategy encourages the scoring module to leverage both the raw inputs and the generated CoT, preventing the reward prediction from being dominated by either source.
During the RL stage, we employ a dual-objective optimization approach. The GRPO loss refines the reasoning quality of the CoT, while an auxiliary BT loss continuously calibrates the scoring module. This decoupled dual-objective explicitly isolates the reward optimization from the high-variance policy updates of the reasoning chain. By ensuring the scoring module receives a direct gradient for the reward value via the BT loss, DeScore achieves superior optimization stability and faster convergence while simultaneously refining the model's reasoning capabilities.
Empirical evaluation shows that DeScore significantly outperforms existing discriminative and generative video reward models in terms of training efficiency and generalization performance. Our contribution can be summarized as follows:

\begin{itemize}[leftmargin=*]
    \item We introduce a decoupled video reward modeling paradigm that separates CoT reasoning from final reward prediction, combining the interpretability and generalization benefits of CoT reasoning, and maintains optimization stability and efficiency.
    
    \item We propose DeScore, a training-efficient and generalizable video reward model built on this paradigm with a two-stage framework: a discriminative cold start with random masking and a dual-objective RL stage that separates reasoning refinement from reward calibration.
    
    \item  Extensive experiments demonstrate that DeScore consistently outperforms state-of-the-art (SOTA) baselines, achieving stronger OOD generalization and higher training efficiency. DeScore also proves effective for post-training, leading to improved generated video quality.
\end{itemize}
    
    

\vspace{-4mm}
\section{Related Work}
\vspace{-2mm}
\noindent \textbf{Video Reward Model.} Existing video reward models mainly follow two paradigms. Discriminative methods \cite{he2024videoscore, liu2025improving} regress scalar rewards from MLLM features using objectives such as MSE or Bradley-Terry (BT) loss \cite{bradley1952rank, rao1967ties}. Although these objectives provide stable optimization, the lack of explicit reasoning makes such models prone to shortcut learning \cite{ye2025rectifying} and heavily reliant on large-scale data to achieve strong generalization. Generative methods formulate reward modeling as next-token prediction. Early works \cite{xu2026visionreward, wang2025unified} directly generated scores, lacking the reasoning capacity to handle complex scenarios. Subsequent methods \cite{wang2025unified2, he2025videoscore2, wang2024lift, wang2025vr} introduced CoT through two-stage training, \ie SFT followed by RL \cite{DeepSeekMath}. Although CoT improves interpretability and generalization, these models often suffer from training instability because reasoning and scoring are coupled within a single sampling chain. Some methods \cite{wu2025rewarddance} instead use token probabilities as rewards, but their reliance on reference videos or pairwise comparisons limits practical applicability. To address these limitations, we propose DeScore, which decouples reasoning from scoring through a ``Think-then-Score'' process, achieves both robust preference alignment and training efficiency.

\noindent \textbf{Reinforcement Learning.} 
Reinforcement learning (RL) has recently achieved strong performance across a wide range of MLLM tasks \cite{hurst2024gpt, bai2025qwen3, jaech2024openai, GPT-5, Gemini-2.5-pro}, substantially improving visual reasoning and understanding \cite{r1_reward, vlmr1, internvl3, llavar1}. Much of this progress has been driven by Group Relative Policy Optimization (GRPO) \cite{DeepSeekMath, DeepSeekR1}, which estimates advantages from the relative rewards of multiple responses to the same input. By removing the need for a separate critic model \cite{ppo}, GRPO makes RL optimization more scalable for MLLMs.  However, recent studies have identified important optimization bottlenecks in GRPO. \cite{yu2025dapo} show that ineffective prompts can produce response groups that are uniformly correct or uniformly incorrect, thereby weakening effective gradient signals and increasing training variance. Meanwhile, \cite{he2025vl} empirically show that the gradient variance of GRPO grows with sequence length, which leads training instability. These limitations are directly inherited by CoT-based video reward models \cite{he2025videoscore2, wang2025unified2, wang2025vr}, where reasoning and scoring are coupled in a single sampling chain, causing the final reward prediction to rely heavily on GRPO-based optimization. This motivates us to move beyond the standard GRPO objective and develop a more efficient optimization strategy for video reward modeling.

\vspace{-3mm}
\section{Method}
\vspace{-2mm}
\subsection{Data Collection}
\vspace{-2mm}
\label{sec:data}
We build our preference dataset by captioning diverse real-world videos and using the captions as prompts for multiple T2V models, including Gen-2 \cite{gen22023}, Pika 1.0 \cite{pika2023}, PixVerse (v1/v2) \cite{pixverse2025}, Dreamina \cite{Dreamina2025}, Luma \cite{luma2024}, Gen-3 \cite{gen22024}, and Kling \cite{kling2026}. Human annotators compare the generated pairs along five alignment dimensions: object, dynamics, environment, style, and camera movement, resulting in 22K training pairs and 1469 in-domain evaluation pairs. We then generate stage-specific CoT annotations to support our two-stage training. Qwen3-VL-8B \cite{bai2025qwen3} is used for the discriminative cold-start stage to activate the scoring module, while Gemini-2.5-Pro \cite{Gemini-2.5-pro} provides fine-grained CoTs with sub-dimension scores for the dual-objective RL stage. In both stages, we apply consistency-based filtering by retaining only CoTs whose implied preferences agree with human labels. Further details are provided in Appendix~\ref{supp:data}.

\begin{figure*}[!t]
\centering
\includegraphics[width=\textwidth]{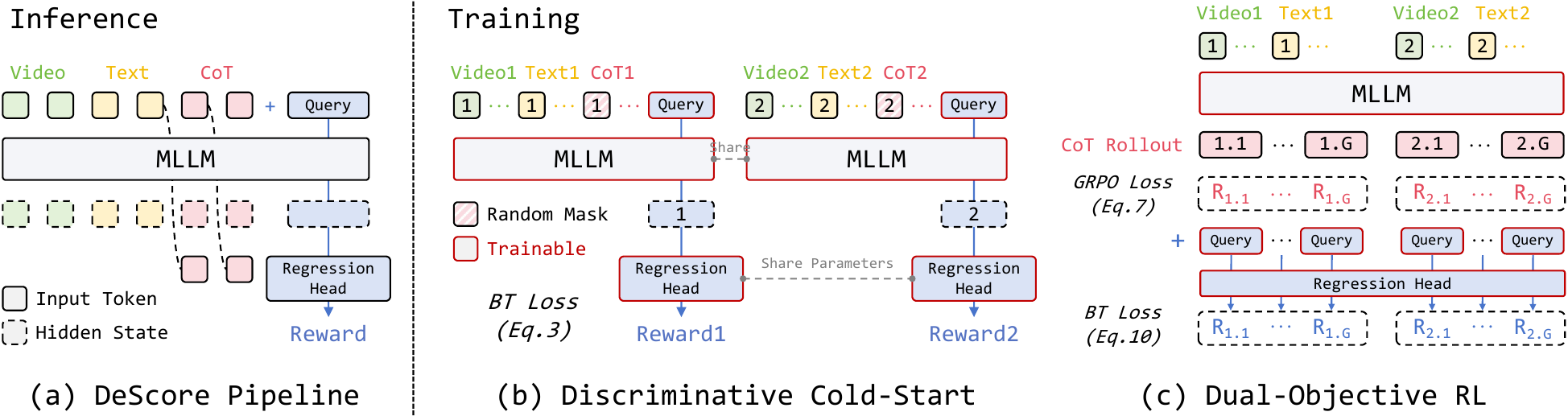}
\vspace{-5mm}
\caption{
\textbf{Our DeScore framework.}
(a) During inference, DeScore first uses an MLLM to generate CoT from the multi-modal input, then appends a learnable query token whose last hidden state is projected by a regression head into the final \textcolor{mycharblue}{\textit{video reward}}. Training follows a two-stage paradigm:
(b) In the discriminative cold-start stage, DeScore is trained with BT loss on pre-collected CoT data, where random CoT masking encourages the scoring module to use both multi-modal inputs and reasoning tokens. (c) In the dual-objective RL stage, the GRPO loss optimizes CoT rollouts guided by rule-based \textcolor{mycharpink}{\textit{rollout rewards}}, while the BT loss simultaneously calibrates the \textcolor{mycharblue}{\textit{video reward}}, decoupling reasoning refinement from reward scoring.
}
\label{fig:method}
\end{figure*}
\vspace{-3mm}
\subsection{Reward Model Learning}
\vspace{-2mm}
As illustrated in Figure~\ref{fig:method}, we propose DeScore, a decoupled reward modeling framework that achieves high generalization and training efficiency. DeScore uses Qwen3-VL-8B \cite{bai2025qwen3} as its multi-modal backbone, augmented with a scoring module comprising a learnable query token and a regression head. For a given generative video and the user instruction, the query token follows the generative CoT sequence, aggregating contextual information from multi-modal inputs and reasoning tokens via the MLLM backbone. Its hidden state is then projected by the regression head into a scalar reward. To ensure both reasoning quality and scoring accuracy, the optimization of DeScore adheres to a two-stage training paradigm. First, a discriminative cold start is performed on CoT data to enable the scoring module to effectively extract and aggregate semantic evidence from both multi-modal inputs and CoT, thereby yielding accurate scalar rewards. Subsequently, a dual-objective RL stage refines reasoning with the GRPO loss while calibrating the reward accuracy using the BT loss. User instructions for both training and inference are detailed in Appendix~\ref{supp:user_instrcution}.

\noindent \textbf{Discriminative Cold Start.} 
In this initial stage, our objective is to warm up the MLLM backbone and the scoring module, enabling them to effectively aggregate semantic information from both raw multi-modal inputs and pre-collected CoT data. 
Formally, given a generated video $\boldsymbol{v}$, a text prompt $\boldsymbol{c}$, and its corresponding pre-collected CoT $\boldsymbol{o}$, we construct the input sequence by appending a learnable query token $\texttt{[Reward]}$ to the end:
\begin{equation}
    \mathcal{X} =
    \left(
    \boldsymbol{v},
    \boldsymbol{c},
    \boldsymbol{o},
    \texttt{[Reward]}
    \right).
\end{equation}
The last hidden state of the $\texttt{[Reward]}$ token, $\boldsymbol{h}_\mathrm{reward} \in \mathbb{R}^d$, captures a condensed semantic summary of the multi-modal inputs and the reasoning process. It is then passed through the learnable regression head $\Phi$ to produce the scalar reward $s$:
\begin{equation}
s = \Phi(\boldsymbol{h}_\mathrm{reward}).
\end{equation}
To align the predicted rewards with human preferences, we employ the Bradley-Terry (BT) loss \cite{bradley1952rank}. Given a preference pair from the dataset $\mathcal{D}$, consisting of a winning sample $\boldsymbol{q}^w = \left(\boldsymbol{v}^w, \boldsymbol{c},  \boldsymbol{o}^w\right)$ and a losing sample $\boldsymbol{q}^l = \left(\boldsymbol{v}^l, \boldsymbol{c}, \boldsymbol{o}^l\right)$, the model computes their respective scores $s^w$ and $s^l$ separately with our DeScore. The training objective is defined as:
\begin{equation}\label{eq:bt_loss}
\mathcal{L}_{\mathrm{BT}} = -\mathbb{E}_{(\boldsymbol{q}^w, \boldsymbol{q}^l) \sim \mathcal{D}} \left[ \log \sigma (s^w - s^l) \right],
\end{equation}
where $\sigma(\cdot)$ denotes the sigmoid function.

To ensure that the decoupled scoring module effectively utilizes both the multi-modal video inputs and the generated CoT, preventing the module from relying solely on the CoT, we apply a random masking strategy during training. During training, the CoT $\boldsymbol{o}$ is masked with a probability $p$. In these instances, the reward $s$ is computed solely based on the raw multi-modal inputs $\left(\boldsymbol{v}, \boldsymbol{c}, \boldsymbol{o} = \varnothing \right)$. This mechanism forces DeScore to maintain a strong grounding in the original video features, ensuring that the final reward is a holistic reflection of both visual evidence and logical reasoning, thereby enhancing the robustness of the reward prediction.

\noindent \textbf{Reinforcement Learning with Dual-Objective.} 
In the second stage, we fine-tune the entire model using a dual-objective RL strategy. We employ Group Relative Policy Optimization (GRPO) to refine the reasoning quality of the CoT. However, optimizing solely for CoT quality can lead to ``reward drift'', where the scoring module loses its calibration. To mitigate this, we combine the GRPO objective with an auxiliary BT loss.
During this stage, the model first generates a CoT reasoning sequence $\boldsymbol{o}$ conditioned on the input $\hat{\boldsymbol{q}} = \left(\boldsymbol{v}, \boldsymbol{c}\right)$. The final input sequence for reward prediction is constructed as
\begin{equation}
\mathcal{X} = \left(\boldsymbol{v}, \boldsymbol{c}, \boldsymbol{o}, \texttt{[Reward]}\right).
\end{equation}
The scalar reward $s$ is then computed by passing $\mathcal{X}$ through the MLLM backbone $\Theta$ and the regression head $\Phi$:
\begin{equation} \label{eq:rl_reward}
s = \Phi(\Theta(\mathcal{X})_{\texttt{[Reward]}}) = \Phi(\boldsymbol{h}_\mathrm{reward}),
\end{equation}
where $\boldsymbol{h}_\mathrm{reward}$ is the hidden state of the $\texttt{[Reward]}$ token, aggregating information from both the multi-modal inputs and the generated CoT $\boldsymbol{o}$.

Following the standard GRPO framework, we sample a group of $G$ responses $\{\boldsymbol{o}_1, \boldsymbol{o}_2, \dots, \boldsymbol{o}_G\}$ from the old policy $\pi_{\theta_{\text{old}}}$ for each input $\hat{\boldsymbol{q}}$. The advantage $A_i$ for the $i$-th response is computed by normalizing the rewards within the group:
\begin{align}
A_{i}
= \frac{R(\boldsymbol{o}_i) - \operatorname{mean}(\{ R(\boldsymbol{o}_1), R(\boldsymbol{o}_2), \dots, R(\boldsymbol{o}_G)\})}{\operatorname{std}(\{ R(\boldsymbol{o}_1), R(\boldsymbol{o}_2), \dots, R(\boldsymbol{o}_G)\})}.
\end{align}

Let $\mathcal{Q}$ denote the human preference training set, the GRPO loss can be formulated as:
\begin{align}
\mathcal{L}_{\mathrm{GRPO}}(\theta) 
= &-\mathbb{E}_{\hat{\boldsymbol{q}} \sim \mathcal{Q}, 
\{\boldsymbol{o}_i\}_{i=1}^G \sim \pi_{\theta_{\mathrm{old}}}(\boldsymbol{o} \mid \hat{\boldsymbol{q}})}
\Bigg\{ 
\frac{1}{G} \sum_{i=1}^G 
\frac{1}{|\boldsymbol{o}_i|} 
\sum_{t=1}^{|\boldsymbol{o}_i|} 
\min \left[ 
\frac{\pi_{\theta}(o_{i,t} \mid \hat{\boldsymbol{q}}, \boldsymbol{o}_{i,<t})}
{\pi_{\theta_{\mathrm{old}}}(o_{i,t} \mid \hat{\boldsymbol{q}}, \boldsymbol{o}_{i,<t})} 
A_{i}, 
\right. \nonumber \\
&\left.
\operatorname{clip}\left( 
\frac{\pi_{\theta}(o_{i,t} \mid \hat{\boldsymbol{q}}, \boldsymbol{o}_{i,<t})}
{\pi_{\theta_{\mathrm{old}}}(o_{i,t} \mid \hat{\boldsymbol{q}}, \boldsymbol{o}_{i,<t})}, 
1-\epsilon, 
1+\epsilon 
\right) 
A_{i} 
\right] 
\Bigg\} 
+ \beta D_{\mathrm{KL}} 
\left(
\pi_{\theta} 
\| 
\pi_{\mathrm{ref}}
\right),
\end{align}
where \(t\) indexes the token position in each generated response, \(o_{i,t}\) denotes the \(t\)-th token of the \(i\)-th response \(\boldsymbol{o}_i\), and \(\boldsymbol{o}_{i,<t}\) denotes the preceding token sequence used as the autoregressive context. The clipping threshold \(\epsilon\) bounds the importance sampling ratio, while \(\beta\) controls the KL regularization strength to prevent the optimized policy \(\pi_\theta\) from deviating excessively from the reference policy \(\pi_{\mathrm{ref}}\). To improve CoT generation, the composite reward \(R(\boldsymbol{o}_i)\) is designed with three components:
\begin{equation}
R(\boldsymbol{o}_i)
=
\lambda_1 R_{\mathrm{fmt}}(\boldsymbol{o}_i)
+
\lambda_2 R_{\mathrm{qual}}(\boldsymbol{o}_i)
+
\lambda_3 R_{\mathrm{len}}(\ell_i),
\end{equation}
where $\lambda_{1}, \lambda_{2}$, and $\lambda_{3}$ are the trade-off weights, and
$\ell_i = |\boldsymbol{o}_i|$ denotes the length of the generated CoT.

\begin{itemize}[leftmargin=*]
\item \textit{Format Reward} ($R_{\mathrm{fmt}}$): Assigns $1$ if the output strictly follows the $\texttt{<think></think>}$ structure and provides a JSON-formatted sub-dimension score, otherwise $0$.
\item \textit{Quality Reward} ($R_{\mathrm{qual}}$): Measures the accuracy of the predicted sub-dimension scores against ground-truth labels: $R_{\mathrm{qual}} = N_{\mathrm{correct}} / N_{\mathrm{total}}$.
\item \textit{Length Reward} ($R_{\mathrm{len}}$): Encourages detailed reasoning while penalizing excessive verbosity or extreme brevity:
\begin{equation}
R_{\mathrm{len}}(\ell) =
\begin{cases}
0, & \ell < 500, \\[0pt]
0.2 \times \left\lfloor {\ell}/{500} \right\rfloor,
& 500 \le \ell < 2000, \\[0pt]
1, & \ell \ge 2000.
\end{cases}
\end{equation}
\end{itemize}

In addition to the GRPO objective, we apply an auxiliary BT loss to calibrate the final reward, ensuring that improvements in CoT quality consistently translate into gains in overall model performance. Given a training pair $(\hat{\boldsymbol{q}}^w, \hat{\boldsymbol{q}}^l)$ consisting of a winning and a losing sample, we generate their respective CoT rollouts, denoted as $\mathcal{O}^w = \{\boldsymbol{o}_1^w, \dots, \boldsymbol{o}_G^w\}$ and $\mathcal{O}^l = \{\boldsymbol{o}_1^l, \dots, \boldsymbol{o}_G^l\}$. While the GRPO loss is computed based on these rollouts, each response $\boldsymbol{o}_i^j$ (where $j \in \{w, l\}$) is also used to construct the input sequence $\mathcal{X}_i^j = \left(\hat{\boldsymbol{q}}^j, \boldsymbol{o}_i^j, \texttt{[Reward]}\right)$. 
We then compute the scalar reward $s_i^j$ for each rollout according to Eq.~\ref{eq:rl_reward}. The final auxiliary BT loss is defined as:
\begin{equation}
\mathcal{L}_{\mathrm{BT}}^{\mathrm{aux}} = -\mathbb{E}_{(\hat{\boldsymbol{q}}^w, \hat{\boldsymbol{q}}^l) \sim \mathcal{D}} \left[ \frac{1}{G} \sum_{i=1}^{G} \log \sigma (s^w_i - s^l_i) \right].
\end{equation}

To integrate dual objectives, the final training loss for this stage is formulated as:
\begin{equation}
\mathcal{L}_{\mathrm{total}} = \mathcal{L}_{\mathrm{GRPO}} + \alpha \mathcal{L}_{\mathrm{BT}}^{\mathrm{aux}},
\end{equation}
where $\alpha$ is a balancing coefficient to align gradient scales. This decoupled design ensures the final reward remains grounded in a stable regression rather than dominated by a coupled sampling chain.

\vspace{-3mm}
\subsection{Inference}
\vspace{-2mm}
During inference, DeScore evaluates videos via a two-step ``think-then-score'' procedure. Given a test generative video $\boldsymbol{v}$ and user prompt $\boldsymbol{c}$, the backbone $\Theta$ first autoregressively generates a detailed CoT $\boldsymbol{o}$ to analyze video quality. Subsequently, the query token $\texttt{[Reward]}$ is appended to form the sequence $\mathcal{X}_{\mathrm{inf}} = \left(\boldsymbol{v}, \boldsymbol{c}, \boldsymbol{o}, \texttt{[Reward]}\right)$. The MLLM backbone $\Theta$ processes the input sequence
$\mathcal{X}_{\mathrm{inf}}$,
and the resulting hidden state of the $\texttt{[Reward]}$ token, $\boldsymbol{h}_{\mathrm{reward}}$, is fed into the regression head $\Phi$ to produce the scalar reward $s$, integrating information from both the multi-modal inputs and the generated CoT:
\begin{equation}
s = \Phi(\Theta(\mathcal{X}_{\mathrm{inf}})_{\texttt{[Reward]}}).
\end{equation}
By decoupling scoring from reasoning, DeScore harnesses the interpretability and generalization introduced by CoT reasoning while shielding the training process from the optimization instability of a coupled sampling chain.

\section{Experiments}
\subsection{Experimental Setups.}
\noindent \textbf{Implementation.}
We use Qwen3-VL-8B \cite{bai2025qwen3} as the backbone of DeScore. In the discriminative cold-start stage, the model is fine-tuned with LoRA (rank 64) for two epochs using AdamW, with a learning rate of $2\times10^{-6}$, weight decay of 0.01, and batch size of 32. The resulting checkpoint is then used to initialize the RL stage. In the RL stage, we optimize DeScore with GRPO and auxiliary BT losses, using coefficients of 1.0 and 0.005, respectively. GRPO is trained with a learning rate of $1\times10^{-6}$, group size $G=8$, 65 training steps, a rollout batch size of 128, and a mini-batch size of 32. Video inputs are processed at 2 fps during both training and inference, with more details shown in Appendix~\ref{supp:imple}.


\noindent \textbf{Baselines and Benchmarks.}
We evaluate DeScore against discriminative baselines, including VideoScore \cite{he2025videoscore2} and VideoAlign \cite{liu2025improving}, and generative baselines, including VisionReward \cite{xu2026visionreward}, UnifiedReward \cite{wang2025unified}, UnifiedReward-Thinking \cite{wang2025unified2}, and VideoScore2 \cite{he2025videoscore2}, where the latter two generate CoT before the final reward. Experiments are conducted on an in-domain preference dataset with 1,469 pairs and two OOD benchmarks: GenAI \cite{jiang2024genai}, containing 1.9k low-resolution short-video pairs from early T2V models, and VideoGen-Bench \cite{liu2025improving, zeng2024dawn}, containing 26.5k higher-resolution and longer-video pairs from current SOTA models. Detailed settings are provided in Appendix~\ref{supp:eval}.

\vspace{-2mm}
\subsection{Main Results}
\vspace{-1mm}
\label{exp:main_results}
\begin{table}
\centering
\small  
\renewcommand{\arraystretch}{1.2}  
\caption{\textbf{Main Experiments on Video Reward Benchmarks.} We evaluate preference accuracy on in-domain and OOD benchmarks. \textbf{Bold} and \underline{underlined} values denote the best and second-best results. DeScore achieves the strongest overall performance, demonstrating superior generalization.}
\vspace{1mm}
\resizebox{\hsize}{!}{\begin{tabular}{cccccc}
\toprule
\multicolumn{1}{c|}{}                                 & \multicolumn{1}{c|}{}                                     & \multicolumn{4}{c}{\textbf{OOD Video Reward Benchmark}}                                                                                                                      \\ \cline{3-6} 
\multicolumn{1}{c|}{}                                 & \multicolumn{1}{c|}{}                                     & \multicolumn{2}{c|}{\textbf{GenAI}}                                                             & \multicolumn{2}{c}{\textbf{VideoGen-Bench}}                                \\ \cline{3-6}
\multicolumn{1}{c|}{\multirow{-3}{*}{\textbf{Model}}} & \multicolumn{1}{c|}{\multirow{-3}{*}{\textbf{In-domain Dataset}}} & \textbf{Acc w/ Tie} & \multicolumn{1}{c|}{\textbf{Acc w/o Tie}}    &  \textbf{Acc w/ Tie} & \textbf{Acc w/o Tie}    \\ \hline
\multicolumn{6}{c}{\textit{Discriminative Video Reward Model}}                                                                                                                                                                                                                                          \\ \midrule
\multicolumn{1}{c|}{VideoScore \cite{he2024videoscore}}                       & \multicolumn{1}{c|}{0.552}                                & 0.490                                                     & \multicolumn{1}{c|}{0.720}          & 0.372                                                     & 0.503          \\
\multicolumn{1}{c|}{VideoAlign \cite{liu2025improving}}                       & \multicolumn{1}{c|}{\underline{0.642}}                                & 0.494                                                     & \multicolumn{1}{c|}{\underline{0.728}}          & \cellcolor[HTML]{FFFFFF}\underline{0.538}                            & \underline{0.722}          \\ \midrule
\multicolumn{6}{c}{\textit{Generative Video Reward Model}}                                                                                                                                                                                                                                                \\ \midrule
\multicolumn{1}{c|}{VisionReward \cite{xu2026visionreward}}                     & \multicolumn{1}{c|}{0.571}                                & {\underline{0.525}}                                               & \multicolumn{1}{c|}{0.724}          & 0.465                                                     & 0.611          \\
\multicolumn{1}{c|}{UnifiedReward \cite{wang2025unified}}                    & \multicolumn{1}{c|}{0.492}                                & 0.458                                                     & \multicolumn{1}{c|}{0.686}         & 0.303                                                     & 0.564          \\
\multicolumn{1}{c|}{UnifiedReward-Thinking \cite{wang2025unified2}}             & \multicolumn{1}{c|}{0.578}                                & \textbf{0.548}                                            & \multicolumn{1}{c|}{0.709}         & \cellcolor[HTML]{FFFFFF}0.428                             & 0.582          \\
\multicolumn{1}{c|}{VideoScore2 \cite{he2025videoscore2}}                      & \multicolumn{1}{c|}{0.617}                                & 0.391                                                     & \multicolumn{1}{c|}{0.616}          & 0.301                                                     & 0.497          \\ \midrule
\multicolumn{6}{c}{\textit{Our Video Reward Model}}                                                                                                                                                                                                                                                       \\ \midrule
\multicolumn{1}{c|}{\textbf{DeScore (Ours)}}                            & \multicolumn{1}{c|}{\textbf{0.734}}                       & 0.504                                                     & \multicolumn{1}{c|}{\textbf{0.765}} & \textbf{0.568}                                            & \textbf{0.768} \\ \bottomrule
\end{tabular}}
\label{main_exp}
\end{table}

\begin{figure*}[t]
\centering
\includegraphics[width=\textwidth]{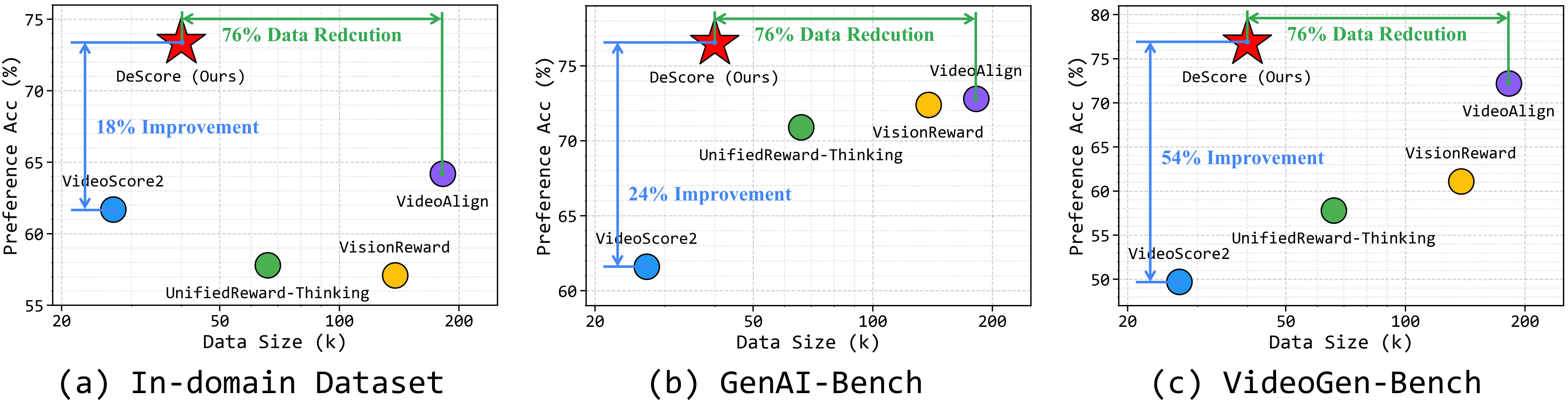}
\vspace{-6mm}
\caption{\textbf{Performance vs. Training Data Size.} DeScore (red star) consistently outperforms existing models by a large margin while requiring only a fraction of the training data, highlighting its extreme training efficiency and robust semantic understanding.}
\vspace{-3mm}
\label{fig:efficient}
\end{figure*}

\noindent \textbf{High Generalization.}
To evaluate the reward accuracy of DeScore, we compare it with several state-of-the-art baselines across both in-domain and OOD benchmarks. As shown in Table~\ref{main_exp}, DeScore outperforms both discriminative and generative methods on the in-domain preference dataset, and this advantage consistently extends to OOD settings. On GenAI-Bench, although UnifiedReward-Thinking achieves a slightly higher Acc w/ Tie, DeScore obtains the best result on the more informative Acc w/o Tie metric (0.765), demonstrating strong performance on videos generated by early-stage T2V models. On VideoGen-Bench, DeScore reaches 0.768 in Acc w/o Tie, substantially outperforming the strongest discriminative baseline, VideoAlign (0.722), and the best generative baseline (0.582). These results show that DeScore achieves state-of-the-art performance across most benchmarks, validating the effectiveness of our ``think-then-score'' paradigm in combining the interpretability of generative reasoning with strong generalization.
\begin{figure*}[!t]
\centering
\includegraphics[width=\textwidth]{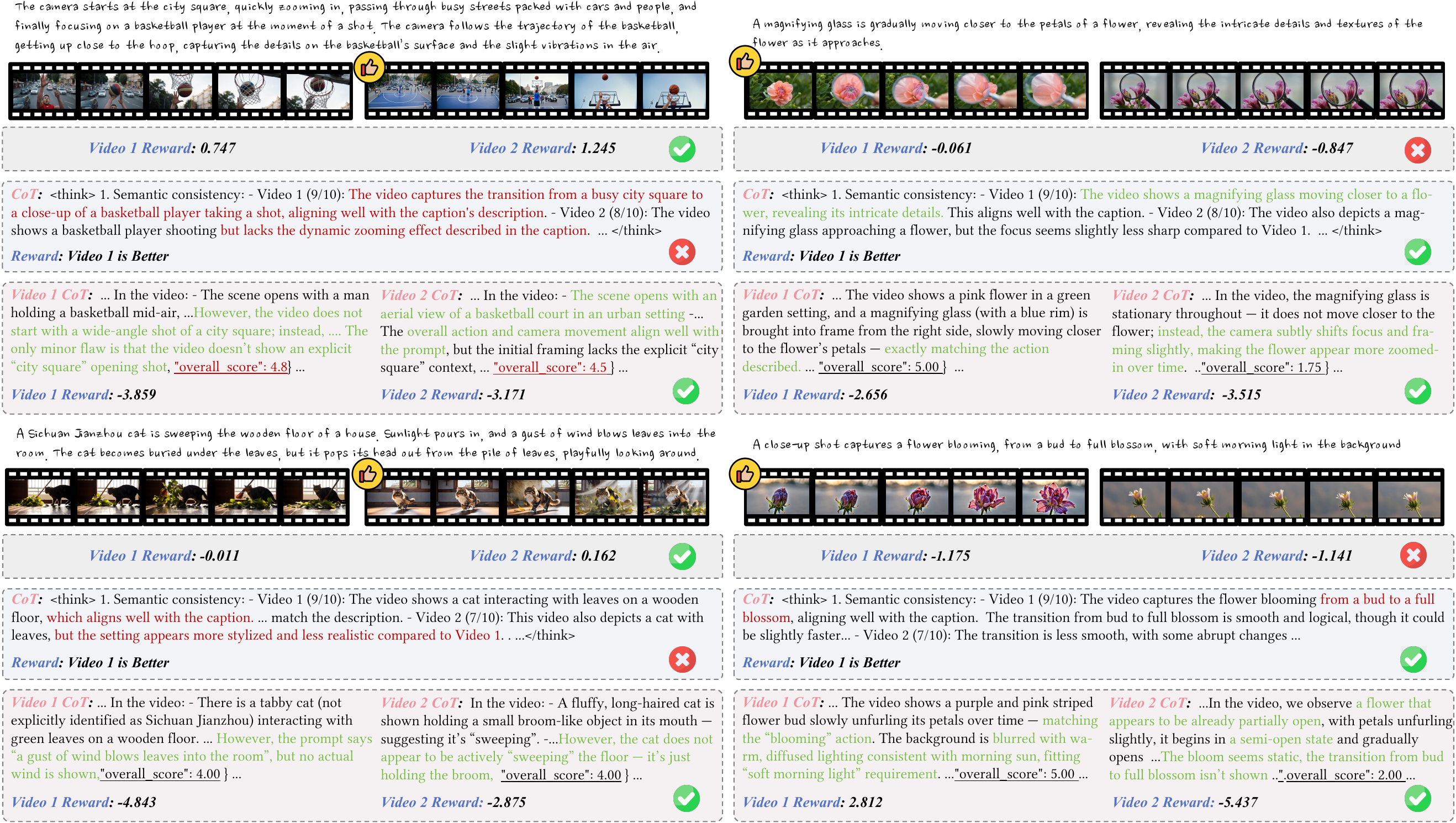}
\vspace{-4mm}
\caption{\textbf{Qualitative Comparison of Different Video Reward Models.} 
We compare \mybox{mygray}{VideoAlign}, 
\mybox{myblue}{UnifiedReward-Thinking}, 
and \mybox{mypink}{our DeScore} 
with \textcolor{mygreen}{high-} and \textcolor{myred}{low-}quality reasoning 
within these responses. DeScore consistently yields accurate rewards and robust 
reasoning across varied prompts, demonstrating its superior interpretability 
and generalization.}
\vspace{-3mm}
\label{fig:case}
\end{figure*}

\noindent \textbf{Training Efficiency.} As shown in Figure~\ref{fig:efficient}, DeScore achieves a superior performance-to-data trade-off across all evaluated benchmarks, outperforming SOTA reward models while reducing training data by 76$\%$. These results show that the decoupled ``think-then-score'' paradigm improves sample efficiency by leveraging fine-grained semantic rationales, bypassing the heavy data dependency of traditional discriminative video reward models.

\noindent \textbf{Qualitative Performance.} As shown in Figure~\ref{fig:case}, we compare DeScore with two representative baselines, the discriminative model VideoAlign and the generative model UnifiedReward-Thinking, across different scenarios including camera motion (row 1) and dynamic motion (row 2). DeScore consistently performs well in all cases. In particular, UnifiedReward-Thinking often fails when the generated CoT is of low quality, since its final score is directly coupled with the reasoning path. In contrast, DeScore decouples reasoning from scoring and applies random masking during Stage 1, encouraging the scoring module to jointly leverage multimodal inputs and CoT rather than relying solely on generated text. This yields two clear advantages: (1)~\underline{\textit{error tolerance}}, where DeScore remains accurate even with imperfect CoT (\eg top-left case), and (2)~\underline{\textit{fine-grained discrimination}}, where it can still produce differentiated scores when generative models output identical reward tokens (\eg bottom-left case). These results show that DeScore effectively combines reasoning interpretability with robust reward prediction.

\vspace{-1mm}
\subsection{Ablation Study}
\vspace{-1mm}
\begin{table}
  \centering
  \renewcommand{\arraystretch}{1.3}  
  \caption{\textbf{Ablation Study of DeScore Components across Different Training Stages.} We investigate the impact of CoT and Random Masking during Stage 1 (top), and the contributions of Cold Start and the auxiliary BT Loss to Stage 2 (bottom). \textbf{Bold} values denote the best performance in each stage.}
  \vspace{1mm}
\resizebox{\hsize}{!}{\begin{tabular}{c|cc|c|cccc}
\hline
\multirow{3}{*}{\textbf{Training Satge}} & \multicolumn{2}{c|}{\multirow{2}{*}{\textbf{Setting}}} & \multirow{3}{*}{\textbf{In-domain Dataset}} & \multicolumn{4}{c}{\textbf{OOD Video Reward Benchmark}}                                                       \\ \cline{5-8} 
                                         & \multicolumn{2}{c|}{}                                  &                                     & \multicolumn{2}{c|}{\textbf{GenAI}}                             & \multicolumn{2}{c}{\textbf{VideoGen-Bench}} \\ \cline{2-3} \cline{5-8} 
                                         & \textbf{CoT}           & \textbf{Random Mask}          &                                     & \textbf{Acc w/ Tie} & \multicolumn{1}{c|}{\textbf{Acc w/o Tie}} & \textbf{Acc w/ Tie}  & \textbf{Acc w/o Tie} \\ \hline
Stage 1 (Discriminative Version)       & $\times$                      & $\times$                             & 0.588                               & 0.379               & \multicolumn{1}{c|}{0.585}                & 0.427                & 0.589                \\
Stage 1      & \checkmark                    & $\times$                             & 0.615                               & 0.417               & \multicolumn{1}{c|}{0.636}                & 0.480                & 0.654                \\
Stage 1 (Default)               & \checkmark                    & \checkmark                           & \textbf{0.656}                      & \textbf{0.449}      & \multicolumn{1}{c|}{\textbf{0.685}}       & \textbf{0.489}       & \textbf{0.672}       \\ \hline
                                         & \multicolumn{2}{c|}{\textbf{Setting}}                           & \multirow{2}{*}{\textbf{In-domain Dataset}}          & \multicolumn{2}{c|}{\textbf{GenAI}}                                      & \multicolumn{2}{c}{\textbf{VideoGen-Bench}} \\ \cline{2-3} \cline{5-8} 
                                         & \textbf{Cold Start}             & \textbf{BT Loss}                       &                                     & \textbf{Acc w/ Tie} & \multicolumn{1}{c|}{\textbf{Acc w/o Tie}} & \textbf{Acc w/ Tie}  & \textbf{Acc w/o Tie} \\ \hline
Stage 2 (Generative Version)               & $\times$                      & $\times$                             & 0.683                               & 0.461               & \multicolumn{1}{c|}{0.697}                & 0.526                & 0.712                \\
Stage 2                & \checkmark                    & $\times$                             & 0.691                               & 0.471               & \multicolumn{1}{c|}{0.754}                & 0.445                & 0.648                \\
Stage 2             & $\times$                      & \checkmark                           & 0.720                               & 0.491               & \multicolumn{1}{c|}{0.748}                & 0.547                & 0.741                \\
Stage 2 (Default)                & \checkmark                    & \checkmark                           & \textbf{0.734}                      & \textbf{0.504}      & \multicolumn{1}{c|}{\textbf{0.765}}       & \textbf{0.568}       & \textbf{0.768}       \\ \hline
\end{tabular}}
\vspace{-3mm}
\label{tab:ablation}
\end{table}
To evaluate the contribution of each component and setting in DeScore, we perform ablation studies on our in-domain dataset and two OOD benchmarks. The results are summarized in Table~\ref{tab:ablation}.

\noindent \textbf{Effectiveness of Stage 1 Components.} Table~\ref{tab:ablation} (top) illustrates the impact of the CoT and random masking mechanism during the discriminative cold-start phase. Incorporating CoT leads to a significant performance boost across all benchmarks, achieving an accuracy improvement of $2.7\%$ (from 0.588 to 0.615) on the in-domain dataset, alongside substantial gains on OOD benchmarks (where Acc w/o Tie increases by $5.1\%$ on GenAI and $6.5\%$ on VideoGen-Bench). This confirms that explicit rationales improve semantic understanding and reward prediction. Random masking further improves performance on the in-domain dataset ($0.615 \to 0.656$), GenAI ($0.636 \to 0.685$), and VideoGen-Bench ($0.654 \to 0.672$). To validate our hypothesis that random masking encourages the model to jointly leverage multi-modal inputs and CoT, rather than over-relying on reasoning tokens, we visualize the top 150 tokens receiving the highest attention from the final reward query. As shown in Figure~\ref{fig:visual} in Appendix~\ref{supp:visual}, random masking causes the query token to attend to both multi-modal inputs and CoT, rather than relying solely on reasoning tokens with better reward performance.

\noindent \textbf{Effectiveness of Stage 2 Components.} Table~\ref{tab:ablation} (bottom) evaluates the contributions of cold-start initialization and the dual-objective optimization in the reinforcement learning (RL) stage.We observe that training with only the GRPO loss leads to a significant performance drop, with accuracy on VideoGen-Bench declining from $0.768$ to $0.648$. We attribute this to a misalignment between reasoning and scoring. While GRPO improves reasoning quality, it neglects the scoring module, leading the model to sacrifice reward accuracy for better rationales. In contrast, introducing the auxiliary BT loss effectively calibrates the reward output. This ensures that improvements in CoT quality consistently translate into gains in overall model performance. Furthermore, it is worth noting that even without cold-start initialization, the dual-objective RL training still achieves respectable performance across benchmarks, though it remains slightly lower than the default setting. These results highlight the robustness of our design, confirming that the dual-objective training remains effective even without a highly optimized starting point.

\noindent \textbf{Comparison on Reward Modeling Paradiagm.} 
We compare DeScore with two representative variants: a discriminative version and a generative version, whose training details are provided in Appendix~\ref{supp:ablation}. The discriminative version (Table~\ref{tab:ablation}, row 1) removes CoT and predicts rewards solely from multi-modal inputs using a regression head, similar to VideoAlign. The generative version (Table~\ref{tab:ablation}, bottom row 1) follows the standard two-stage pipeline of SFT and GRPO, and predicts rewards via next-token generation, similar to VideoScore2. Both variants underperform DeScore across all benchmarks. In particular, the discriminative variant performs notably worse, indicating limited generalization and a stronger reliance on data scaling to maintain accuracy. This is further supported by the efficiency analysis in Figure~\ref{fig:efficient}, which shows that it requires substantially more training data to achieve comparable performance. Figure~\ref{fig:intro} (c) further compares the optimization stability of the two variants, showing that training with BT loss yields a smoother and more consistent improvement in preference accuracy, whereas GRPO exhibits pronounced fluctuations. 

\subsection{Improving Video Generation}
\begin{wraptable}{r}{0.50\textwidth}
\vspace{-4.4mm}
\centering
\renewcommand{\arraystretch}{1.3}  
\caption{\textbf{Comparison of Reward Models for Improving Video Generation Quality on VBench.} DeScore consistently improves the quality of generated videos under two different post-training paradigms.}
\vspace{-1mm}
\resizebox{\hsize}{!}{\begin{tabular}{c|ccccc}
\hline
Model                    & SC$\uparrow$    & BC$\uparrow$    & AQ$\uparrow$    & IQ $\uparrow$   & DD $\uparrow$   \\ \hline
Wan-2.1-1.3B             & 0.951 & 0.961 & 0.547 & 0.669 & 0.527 \\ \hline
w/ Longcat-GRPO & 0.969 & 0.973 & 0.645 & 0.706 & 0.541 \\ \hline
w/ Flow-DPO              & 0.969 & 0.972 & 0.615 & 0.700 & 0.542 \\ \hline
\end{tabular}}
\label{tab:video_grpo}
\vspace{-3mm}
\end{wraptable}
To further demonstrate the effectiveness of DeScore for improving video generation, we integrate it into two representative post-training frameworks, Longcat-GRPO \cite{team2025longcat} and Flow-DPO \cite{liu2025improving}, built on Wan-2.1-1.3B \cite{wan2025wan}, and evaluate the resulting models on VBench \cite{huang2024vbench}. Detailed settings are provided in Appendix~\ref{supp:ablation}. As shown in Table~\ref{tab:video_grpo}, DeScore consistently improves generation quality under both frameworks, yielding gains in subject consistency (SC), background consistency (BC), aesthetic quality (AQ), image quality (IQ), and dynamic degree (DD). These results show that DeScore serves as an effective reward model for post-training higher-quality video generators. We further provide qualitative comparisons in Figure~\ref{fig:grpo_case} of Appendix~\ref{supp:video_grpo}, which illustrate that DeScore improves prompt fidelity and video quality across diverse scenarios.

\section{Conclusion}
In this work, we introduced DeScore, a decoupled ``Think-then-Score'' framework for video reward modeling. By decoupling CoT reasoning from final reward prediction, DeScore retains the interpretability and generalization benefits of explicit reasoning while avoiding the optimization instability caused by coupled reasoning and scoring in generative reward models. With a two-stage training framework, DeScore achieves stable and efficient optimization and consistently outperforms existing discriminative and generative baselines on both in-domain and OOD benchmarks, while requiring substantially less training data. Moreover, DeScore further improves generated video quality when applied to post-training and test-time scaling. These results highlight decoupled reasoning and scoring as a promising paradigm for training-efficient and generalizable video reward modeling.

\bibliographystyle{plain}
\bibliography{main}


\appendix

\onecolumn
\section{Analysis on Optimization Direction}
\label{supp:direct}
In generative reward models, the final reward, whether a point-wise score or a pair-wise preference, is formulated as a discrete token $s \in \mathcal{V}$. During the supervised fine-tuning (SFT) stage, optimization heavily relies on the cross-entropy (CE) loss for next token prediction, formulated as $\mathcal{L}_{CE} = -\log \pi_\theta(s^* | \hat{\boldsymbol{q}}, \boldsymbol{o})$, where $s^*$ is the target score token. However, this formulation treats rewards as orthogonal classes, penalizing categorical mismatches while entirely ignoring ordinal distance during reward optimization (\eg, mispredicting a 5 as a 4 incurs a similar penalty as predicting a 1). This fundamental limitation persists into the reinforcement learning (RL) stage. Optimization typically relies on policy gradients (\eg, GRPO), where the gradient with respect to parameters $\theta$ is approximated as $ \hat{g}_{GRPO} \propto A \cdot \nabla\theta \log \pi_\theta(s | \hat{\boldsymbol{q}}, \boldsymbol{o}) $. Here, $A$ represents the advantage, $\hat{\boldsymbol{q}}$ denotes the multi-modal inputs, and $\boldsymbol{o}$ represents the generated CoT. Critically, this categorical treatment lacks numerical directionality. When a predicted reward token is suboptimal, the policy gradient merely suppresses its probability mass without indicating the sign of the error—\ie, whether the score should be increased or decreased to reach the ground truth. This non-directional optimization forces the model to explore the discrete token space blindly, making it difficult to provide clear guidance for distinguishing different quality levels.

Conversely, the Bradley-Terry (BT) loss models the reward as a continuous scalar. For a chosen and rejected pair $(\hat{\boldsymbol{q}}^w, \hat{\boldsymbol{q}}^l)$ with corresponding scores $s_w$ and $s_l$, the loss is formulated as $\mathcal{L}_\mathrm{BT} = -\log \sigma(s^w - s^l)$. The gradients with respect to the continuous rewards are:
\begin{equation}
     \frac{\partial \mathcal{L}_\mathrm{BT}}{\partial s^w} = -(1 - \sigma(s^w - s^l)), \quad \frac{\partial \mathcal{L}_\mathrm{BT}}{\partial s^l} = 1 - \sigma(s^w -s^l) 
\end{equation}
 Unlike policy gradients, the BT loss provides a deterministic, push-and-pull scalar force. The magnitude of this gradient is directly proportional to the current margin error ($1 - \sigma(\cdot)$). Furthermore, this formulation inherently acts as a built-in curriculum learning mechanism: it imposes substantial gradients on hard cases ($s^w \approx s^l$) and adaptively decays for easy cases ($s^w \gg s^l$), thereby ensuring highly efficient optimization.

\section{Analysis on Gradient Variance}
\label{supp:gradient}

Let $\pi_\theta$ denote an autoregressive MLLM policy $\Theta$ parameterized by $\theta$.
Given a multimodal input $\hat{\boldsymbol q}=(\boldsymbol v,\boldsymbol c)$, the policy samples a token sequence $\boldsymbol o_i=(o_{i,1},\dots,o_{i,T_i})$, where $o_{i,t}\in\mathcal V$. 

For the coupled generative paradigm, we consider a simplified policy-gradient form of GRPO. Given a group of $G$ sampled responses $\{\boldsymbol o_i\}_{i=1}^{G}$ for the same input $\hat{\boldsymbol q}$, we focus on one sampled response and omit the sampling-chain index $i$ for notational simplicity. The gradient contribution of a single sampling chain can then be written as
\begin{equation}
  \hat{g}_{\mathrm{GRPO}}
  \propto
  \hat A
  \sum_{t=1}^{T}
  \nabla_\theta
  \log \pi_\theta
  \left(
    o_t
    \mid
    \hat{\boldsymbol q}, \boldsymbol{o}_{<t}
  \right),
  \label{eq:grpo_g}
\end{equation}
where $\hat A$ denotes the group-normalized advantage of the sampled response $\boldsymbol o$. \cite{he2025vl} empirically show that gradient variance increases with the length of the response trajectory. We further analyze this phenomenon from a theoretical perspective.

\subsection*{Theorem: GRPO Gradient Variance Scales as $\Omega(T)$}

We begin by establishing the variance properties of the cumulative score function $\boldsymbol{S}$ under the standard autoregressive generation process, independent of any reward signal.

\begin{lemma}
\label{lem:mds_variance}
Under the autoregressive generation process $\boldsymbol{o} \sim \pi_\theta$, let $\mathcal{H}_t = (\hat{\boldsymbol q}, o_1, \dots, o_t)$ denote the trajectory history up to step $t$. The per-step score function $\boldsymbol{s}_t = \nabla_\theta \log \pi_\theta(o_t \mid \hat{\boldsymbol{q}}, \boldsymbol{o}_{<t})$ forms a Martingale Difference Sequence (MDS) conditioned on the history $\mathcal{H}_{t-1}$. Consequently, the variance of the cumulative score function $\boldsymbol{S} = \sum_{t=1}^T \boldsymbol{s}_t$ satisfies:
\[
  \Var[\boldsymbol{S}] \geq T \cdot G_0,
\]
where $G_0 \triangleq \min_t \mathbb{E}[\|\boldsymbol{s}_t\|^2] > 0$ is the minimum per-step variance.
\end{lemma}

\begin{proof}
Specifically, the score function identity gives:
\begin{align}
    \mathbb{E}[\boldsymbol{s}_t \mid \mathcal{H}_{t-1}]
  &=  \mathbb{E}_{o_t \sim \pi_\theta(\cdot|\hat{\boldsymbol{q}}, \boldsymbol{o}_{<t})}
    \!\bigl[\nabla_\theta \log \pi_\theta(o_t \mid \hat{\boldsymbol{q}}, \boldsymbol{o}_{<t})\bigr] \nonumber \\
  &= \sum_{o_{t}}\pi_\theta(o_t \mid \hat{\boldsymbol{q}}, \boldsymbol{o}_{<t})\nabla_\theta \log \pi_\theta(o_t \mid \hat{\boldsymbol{q}}, \boldsymbol{o}_{<t}) \nonumber \\
  &= \nabla_\theta \sum_{o_t} \pi_\theta(o_t \mid \hat{\boldsymbol{q}}, \boldsymbol{o}_{<t})
  = \nabla_\theta \,1 = 0.
\end{align}

By the MDS property, cross-covariances vanish exactly:
\[
\mathbb E[\langle \boldsymbol{s}_i,\boldsymbol{s}_j\rangle]
=
\mathbb E\left[
\mathbb E[\langle \boldsymbol{s}_i,\boldsymbol{s}_j\rangle \mid \mathcal H_{j-1}]
\right]
=
\mathbb E\left[
\left\langle
\boldsymbol{s}_i,
\mathbb E[\boldsymbol{s}_j\mid \mathcal H_{j-1}]
\right\rangle
\right]
=0.
\]
Therefore, the variance of $\boldsymbol{S}$ decomposes perfectly into the sum of per-step variances:
\begin{equation}
  \Var[\boldsymbol{S}]
  = \Var\!\left[\sum_{t=1}^T \boldsymbol{s}_t\right]
  = \sum_{t=1}^T \Var[\boldsymbol{s}_t]
  \geq \sum_{t=1}^T G_0 = T \cdot G_0.
  \label{eq:var-S}
\end{equation}
\end{proof}

\begin{theorem}
\label{thm:grpo}
Under Lemma~\ref{lem:mds_variance}, the variance of the GRPO gradient estimator satisfies
\[
  \Var\bigl[\hat{g}_{\mathrm{GRPO}}\bigr] \;=\; \Omega(T).
\]
\end{theorem}

\begin{proof}[Proof]
We evaluate the variance of the coupled estimator, under the approximation that $\hat A$ and 
$\nabla_\theta \log \pi_\theta(o_t \mid \hat{\boldsymbol{q}}, \boldsymbol{o}_{<t})$ are independent, we have:
\begin{equation}
   \Var\bigl[\hat{g}_{\mathrm{GRPO}}\bigr] = \Var[\hat A \cdot \boldsymbol{S}] \nonumber \propto  \Var[\boldsymbol{S}] \cdot \mathbb{E}[\hat A^2].
\end{equation}
Finally, by construction of the group relative z-score normalization in GRPO, we have $\mathbb{E}[\hat A^2] = 1$. Therefore:
\[
  \Var\bigl[\hat{g}_{\mathrm{GRPO}}\bigr]
  \;\geq\; T \cdot G_0
  \;=\; \Omega(T). 
\]
\end{proof}

This theorem formally establishes that the gradient variance of the GRPO estimator scales with response length, directly causing the pronounced fluctuations observed in human preference accuracy during GRPO training. While the clipping mechanism in GRPO (i.e., truncating importance weights to $[1-\varepsilon, 1+\varepsilon]$) can partially control the gradient magnitude, it introduces additional optimization bias. This represents a bias-variance tradeoff rather than a fundamental resolution to the $\Omega(T)$ variance growth.

\section{Details on Data Collection}
\label{supp:data}
To construct our training dataset, we first caption collected real-world videos encompassing diverse subjects, dynamics, environments, styles, and camera movements. These captions are subsequently utilized as text prompts for video generation, effectively challenging T2V models to produce high-quality content that maintains strict semantic alignment with text prompt. We employ a suite of T2V models, including Gen-2 \cite{gen22023}, Pika 1.0 \cite{pika2023}, PixVerse (v1/v2) \cite{pixverse2025}, Dreamina \cite{Dreamina2025}, Luma \cite{luma2024}, Gen-3 \cite{gen22024}, and Kling  \cite{kling2026} to generate video pairs based on these prompts. Human annotators then evaluate these pairs, providing preference labels by analyzing alignment across five sub dimensions: object, dynamics, environment, style, and camera movement. In total, we collect 22K video pairs for training and an additional 1469 pairs as an in-domain dataset to benchmark DeScore.

The resulting preference data is then integrated into our two-stage training paradigm, where CoT data are strategically generated and filtered to match specific optimization objectives: (1) Discriminative Cold-Start. In this stage, the priority is to activate the scoring module rather than enforce reasoning quality. We employ Qwen3-VL-8B \cite{bai2025qwen3} to generate the CoTs. To ensure data reliability, we perform consistency-based filtering, retaining only pairs where the preference derived from the CoT aligns with ground-truth human labels. (2) Dual-Objective Reinforcement Learning (RL). To further refine reasoning quality via GRPO, we leverage Gemini-2.5-Pro \cite{Gemini-2.5-pro} to produce high quality and fine-grained CoTs, including sub-dimension scores (covering subjects, dynamics, environment, style, and camera movement) alongside the final preference. Consistent with the first stage, these CoTs are selectively filtered based on preference accuracy to ensure high dataset fidelity.

\section{User Instruction}
\label{supp:user_instrcution}
When evaluating human preference alignment across in-domain datasets and OOD benchmarks, we compare DeScore against state-of-the-art (SOTA) video reward models using their original prompts, which are designed for holistic video quality assessment. For DeScore, we apply the specific user instruction illustrated in Figure~\ref{box:prompt}. To accurately assess whether the generated video faithfully reflects the semantics of the text prompt, our instruction guides the model through a structured reasoning process. Specifically, it first deduces the expected visual elements from the text prompt, and subsequently verifies their presence in the generated video, explicitly assigning scores within the CoT across five fine-grained sub-dimensions: subject, dynamics, camera, environment, and style. Finally, the generated CoT rationale is processed by the dedicated scoring module to output the final scalar reward, providing a robust quantitative evaluation of the video's alignment with the text.
\begin{tcolorbox}[
    breakable,
    width=\textwidth,
    colback=white,
    colframe=black,
    enhanced,
    sharp corners,
    boxrule=1pt,
    drop shadow,
    coltitle=white,
    fonttitle=\bfseries,
    title=User Instruction for Our DeScore,
]
\small
\setlist{noitemsep, topsep=2pt, parsep=0pt, partopsep=0pt}  

\noindent \textbf{Role}: You are a professional Video-Text Alignment Judge. Your job is to strictly evaluate whether a generated video matches the user's text prompt.

\smallskip
\noindent \textbf{Inputs}:
\begin{itemize}
    \item \textbf{Text Prompt}: [Insert the text prompt here]
    \item \textbf{Video}: \textless video\textgreater
\end{itemize}

\noindent \textbf{Step 1: Thinking Process (Chain of Thought)}\\
\textit{You must perform your reasoning in plain text BEFORE generating the final \texttt{<answer>} block. Follow these three stages of thought:}

\begin{enumerate}
    \item \textbf{Prompt Decoding \& Visual Expectation}: Analyze the text prompt to establish a ``Gold Standard'' for evaluation before looking at the video.
    \begin{itemize}
        \item \textbf{Explicit Elements}: List specific entities, actions, texts, or objects mentioned.
        \item \textbf{Implicit/Abstract Translation}: If the prompt contains abstract concepts (e.g., ``loneliness'', ``cinematic'', ``chaos''), translate them into concrete visual indicators (e.g., ``dark lighting'', ``shallow depth of field'', ``fast camera movement'').
    \end{itemize}

    \item \textbf{Visual Evidence Extraction}: Look at the video and describe what you actually see objectively, without forcing a match to the prompt.
    \begin{itemize}
        \item What/Who is the main subject?
        \item What are the exact actions and movements?
        \item Are there any visual artifacts, hallucinations, or typos in generated text?
    \end{itemize}

    \item \textbf{Dimensional Analysis \& Score Calculation}: Compare the Visual Evidence against the Prompt Expectations across 5 dimensions.
    \begin{itemize}
        \item \textbf{Subject (Primary)}: Accuracy of main entities and text generation (spelling).
        \item \textbf{Dynamics (Primary)}: Accuracy of actions, physics, and temporal flow.
        \item \textbf{Camera (Secondary)}: Movement, angles, shot types.
        \item \textbf{Environment (Secondary)}: Background, setting, lighting.
        \item \textbf{Style (Secondary)}: Art style, aesthetic, visual quality.
    \end{itemize}
    \textit{Sub-Dimension Scoring:} Assign \textbf{2} (Perfect), \textbf{1} (Partial/Minor flaws), \textbf{0} (Failure/Hallucination), or \textbf{N/A} (Not applicable).
\end{enumerate}

\noindent \textbf{Step 2: Final Output}\\
After completing your thinking process, output the final scores strictly in the following JSON format enclosed within \texttt{<answer>} tags. Do not include any reasoning inside the JSON.

\begin{tcolorbox}[
    colback=gray!5,
    colframe=gray!30,
    boxrule=0.3pt,
    left=2mm, right=2mm,
    top=1mm, bottom=1mm,
]
\scriptsize
\begin{verbatim}
<answer>
{
  "dimensional_scores": {
    "subject": "0, 1, 2, or N/A",
    "dynamics": "0, 1, 2, or N/A",
    "camera": "0, 1, 2, or N/A",
    "environment": "0, 1, 2, or N/A",
    "style": "0, 1, 2, or N/A"
  },
  "overall_score": 1.00-5.00
}
</answer>
\end{verbatim}
\end{tcolorbox}

\end{tcolorbox}
\captionof{figure}{\textbf{User Instruction for Our DeScore during training and inference.}}
\label{box:prompt}
\section{Detailed Implementation}
\label{supp:imple}
We employ Qwen3-VL-8B \cite{bai2025qwen3} as the backbone for DeScore. In the discriminative cold start stage, the model is fine-tuned on our constructed CoT dataset using LoRA with a rank of 64. The model is trained for two epochs using the AdamW optimizer with a learning rate of $2\times10^{-6}$, a weight decay of 0.01, and a batch size of 32. The resulting checkpoint is then used to initialize the subsequent stage.
In the RL stage, we optimize a dual-objective comprising GRPO and BT losses. To balance the gradient scales between them, the coefficient for the BT-loss is set to 0.005, while the GRPO loss coefficient remains 1.0. We apply GRPO with a learning rate of $1\times10^{-6}$ and a group size of $G=8$. This stage is conducted for 65 steps with a rollout batch size of 128 and a mini-batch size of 32. For video processing during inference and training, the frame rate is set to 2 fps. All experiments are conducted on 8 A800 GPUs.

\section{Detailed Evaluation Setting}
\label{supp:eval}
\noindent \textbf{Baseline.}
To evaluate the capacity of DeScore in assessing diverse generative videos, we compare it against several video reward models. We first consider discriminative models, including VideoScore \cite{he2025videoscore2} and VideoAlign \cite{liu2025improving}, which are trained to predict point-wise rewards using MSE loss and BT loss, respectively. Furthermore, we compare DeScore with state-of-the-art generative models, including VisionReward \cite{xu2026visionreward}, UnifiedReward \cite{wang2025unified}, UnifiedReward-Thinking \cite{wang2025unified2}, and VideoScore2 \cite{he2025videoscore2}. While the first two directly generate reward tokens, the latter two utilize a coupled sampling path that generates the CoT prior to the final reward to enhance interpretability. Specifically, UnifiedReward-Thinking is a pair-wise model trained to identify the superior video from a given pair.

\noindent \textbf{Evaluation Benchmarks.} Comparison experiments are conducted on an in-domain preference dataset consisting of 1,469 sample pairs. This dataset is a held-out subset partitioned from our original data source and remains unseen during training. To evaluate the generalization of DeScore, we further benchmark it on two Out-of-Distribution (OOD) suites: GenAI \cite{jiang2024genai} and VideoGen-Bench \cite{liu2025improving, zeng2024dawn}. The former comprises 1.9k pairs generated by 10 early-stage T2V models across 508 prompts, typically characterized by lower resolutions ($\sim$320$\times$512) and shorter durations (2s-2.5s). In contrast, the latter contains 26.5k video pairs from 12 diverse open- and closed-source models across 420 prompts, featuring higher resolutions (up to 576$\times$1024), longer durations (4s--6s), and significantly improved visual quality. These two benchmarks represent early-generation and current state-of-the-art video models, respectively, providing a comprehensive assessment of DeScore across diverse generative scenarios.

\noindent{\textbf{Evaluation Metrics.}} Across diverse video reward benchmarks, we employ preference accuracy to assess the performance of DeScore, reporting both accuracy with ties (Acc w/ Tie) and without ties (Acc w/o Tie). Specifically, Acc w/o Tie directly compares the point-wise scores within each video pair, assigning the preference to the higher-scoring video. For Acc w/ Tie, if the absolute score difference between the two videos falls below a predefined threshold, the pair is classified as a tie. Furthermore, because the preference labels in both our training and evaluation datasets are fundamentally grounded in semantic fidelity to the text prompt, we specifically report the Text Alignment (TA) preference accuracy when evaluating DeScore on VideoGen-Bench. For state-of-the-art (SOTA) baselines, we utilize their corresponding task-specific scores for a fair comparison. Specifically, we evaluate VideoAlign using its dedicated text alignment (TA) sub-score. For UnifiedReward and VideoScore2, we explicitly apply their ``text-to-video alignment'' scores for this dimension, while for all other remaining baseline models, we utilize their overall scores for performance comparison.

\section{Training Details of Ablation Study}
\label{supp:ablation}

\noindent \textbf{Discriminative Version.} This version shares a similar architecture with VideoAlign but employs Qwen3-VL-8B \cite{bai2025qwen3} as the backbone for feature extraction. Unlike DeScore, the learnable query token is appended directly after the multi-modal inputs without any explicit CoT generation. To ensure a fair comparison, it is trained on the exact same dataset and uses the identical configuration as the discriminative cold start stage of DeScore. Specifically, this discriminative baseline is fine-tuned via LoRA (rank of 64) for two epochs, using the AdamW optimizer with a learning rate of $2 \times 10^{-6}$, a weight decay of 0.01, and a batch size of 32.

\noindent \textbf{Generative Version.} This version follows the standard training paradigm for generative reward models, which first conducts supervised fine-tuning (SFT) and subsequently refines the model's capabilities via GRPO. It is formulated as a pair-wise model, given that only relative preferences, rather than absolute video reward scores, are accessible in the training dataset. The CoT data used for SFT is generated by Gemini-2.5-Pro following the approach detailed in Section~\ref{sec:data}. During the GRPO training stage, we incorporate additional data to ensure robust model convergence. The model is optimized using the AdamW optimizer with a weight decay of 0.01 and a learning rate of $1 \times 10^{-6}$. Specifically, the GRPO training utilizes a rollout group size of $G=8$, a rollout batch size of 256, and an update mini-batch size of 64.

\section{Additional Experiments}
\subsection{Visualization}
\label{supp:visual}
\begin{figure*}[!t]
\centering
\includegraphics[width=\textwidth]{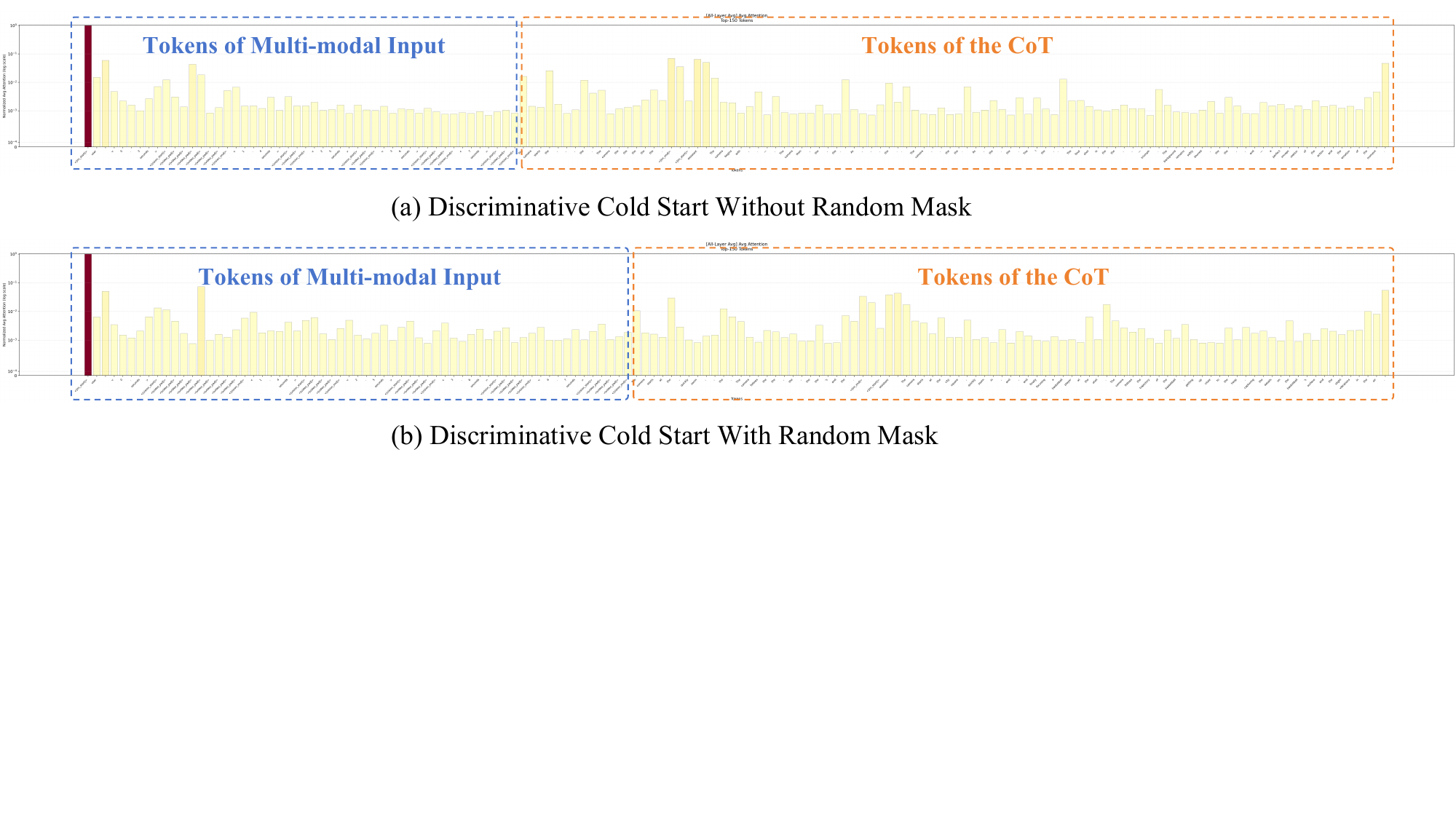}
\vspace{-4cm} 
\caption{\textbf{Visualization of the top 150 tokens most attended to by the \texttt{[Reward]} token.}
Random masking encourages the model to attend more extensively to multi-modal input tokens, preventing the final reward prediction from relying solely on the CoT.}
\vspace{-5mm}
\label{fig:visual}
\end{figure*}
To validate our hypothesis that random masking encourages the
model to jointly leverage multi-modal inputs and CoT, rather than over-relying on reasoning tokens, we visualize the top 150 tokens receiving the highest attention from the final reward query. As shown in Figure~\ref{fig:visual}, random masking causes the query token to attend to both multi-modal
inputs and CoT, rather than relying solely on reasoning tokens, which contributes to better reward performance.

\subsection{Improving Video Generation}
\label{supp:video_grpo}
\begin{figure*}[!t]
\centering
\includegraphics[width=\textwidth]{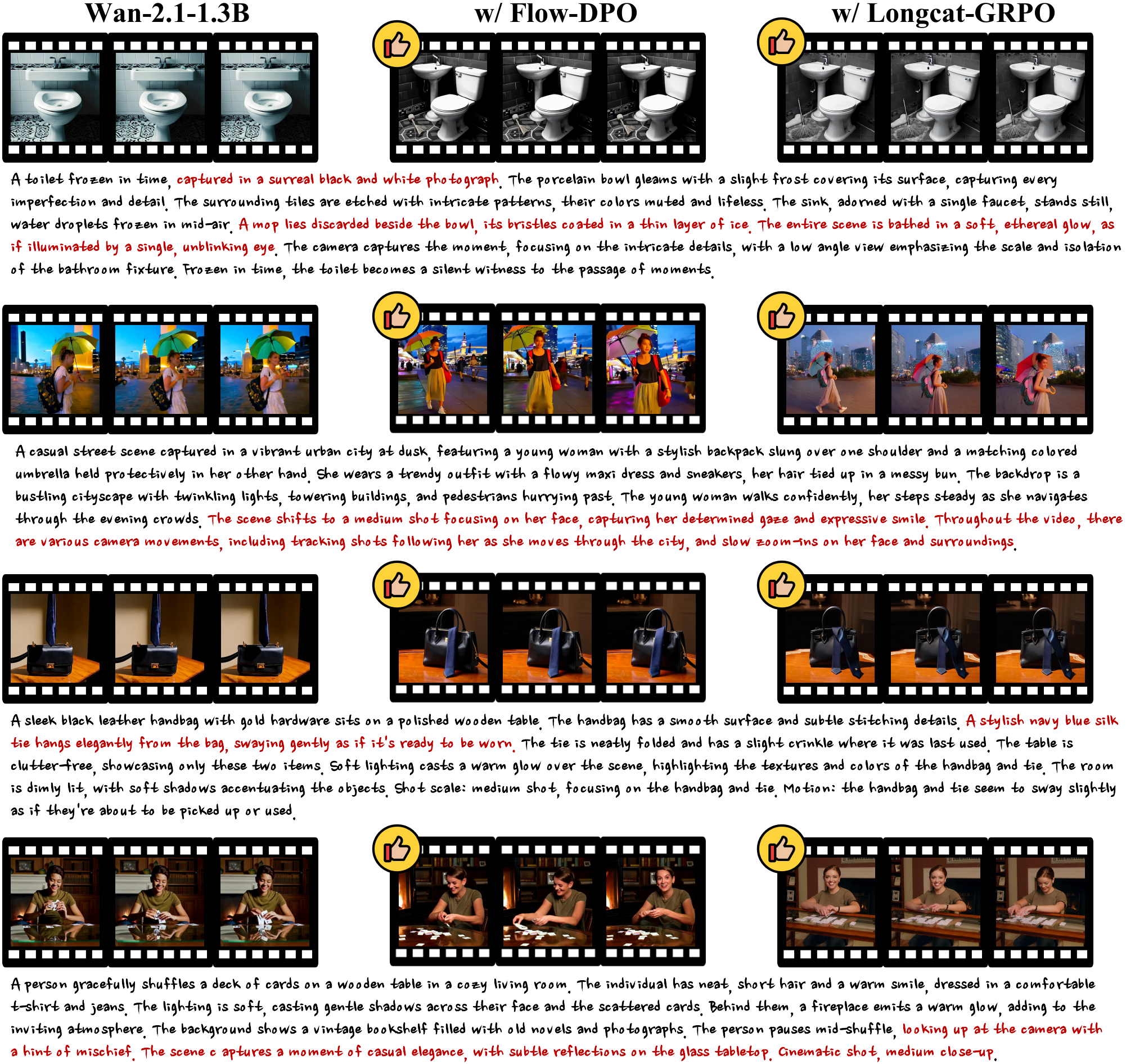}
\vspace{-4mm}
\caption{\textbf{Qualitative examples of improved video generation with DeScore.}}
\vspace{-7mm}
\label{fig:grpo_case}
\end{figure*}
To further demonstrate the effectiveness of DeScore in improving generated video quality, we integrate it into two representative post-training paradigms, Longcat-GRPO \cite{team2025longcat} and Flow-DPO \cite{liu2025improving}, on Wan-2.1-1.3B \cite{wan2025wan}, and evaluate the resulting models on VBench \cite{huang2024vbench}. We sample 10K prompts from those used to construct our preference video dataset. For Longcat-GRPO, we follow the official setting in its technical report. For Flow-DPO, we generate videos with Wan-2.1-1.3B and form positive and negative pairs according to reward scores from the corresponding reward model, with the pairs periodically updated during training following the official setting.

As shown in Table~\ref{tab:video_grpo}, DeScore consistently improves the quality of generated videos under both post-training paradigms. We further provide qualitative examples in Figure~\ref{fig:grpo_case}. Compared with the baseline Wan-2.1-1.3B, post-training with DeScore significantly improves video quality across multiple aspects, including subject fidelity, camera motion, spatial relationships, and dynamics. While the baseline model often fails to generate videos that faithfully match the text prompt, both Flow-DPO and Longcat-GRPO equipped with DeScore produce videos that better capture the semantic content of the prompt.

\section{Limitation and Future Work}
\label{limit}
Although DeScore demonstrates strong generalization across diverse scenarios, it is primarily evaluates whether generated videos are faithful to the input text prompt. Consequently, it may be less effective at capturing motion implausibility and visual artifacts. In future work, we plan to extend our decoupled paradigm to multi-dimensional video reward modeling, aiming to build a more comprehensive reward model for advancing video generation.

\section{Ethics statement}
\label{ethics}
Our work focuses on algorithmic improvements for video reward modeling, with the goal of more accurately assessing the quality of generated videos. The training data consist of videos generated by open-source and closed-source generative models, and the evaluation is conducted on publicly available benchmarks, including GenAI, VideoGen-Bench, and VBench. No data are collected from human subjects. We do not anticipate any direct, immediate, or negative societal impacts arising from this research.

\section{Reproducibility statement}
\label{code}
To ensure reproducibility, we provide detailed descriptions of datasets, training configurations, and hyper-parameters in both the main text and supplementary materials. We will provide our code to facilitate reproducibility.



\end{document}